\documentclass{article}

\usepackage{microtype}
\usepackage{graphicx}
\usepackage{booktabs} 
\usepackage{wrapfig}
\usepackage{amssymb}
\usepackage{mathtools}
\usepackage{multirow, multicol}
\usepackage{tcolorbox}
\usepackage{amsmath}
\usepackage{amssymb}
\usepackage{mathtools}
\usepackage{amsthm}
\usepackage[table]{xcolor}
\usepackage[T1]{fontenc}

\definecolor{myblue}{HTML}{000089}
\definecolor{myred}{HTML}{FF0000}

\usepackage{hyperref}

\usepackage[preprint]{icml2026}

\usepackage[capitalize,noabbrev]{cleveref}

\theoremstyle{plain}
\newtheorem{theorem}{Theorem}[section]
\newtheorem{proposition}[theorem]{Proposition}

\theoremstyle{definition}

\theoremstyle{remark}

\usepackage[textsize=tiny]{todonotes}

\icmltitlerunning{Submission and Formatting Instructions for ICML 2026}

\begin{document}

\twocolumn[
  \icmltitle{On the Necessity of Attention--FFN Split in Vision Transformers}



  \icmlsetsymbol{equal}{*}

\begin{icmlauthorlist}
  \icmlauthor{Junhyeok Kim}{yyy}
  \icmlauthor{Jinyeong Kim}{yyy}
  \icmlauthor{Jae Wan Park}{yyy}
  \icmlauthor{Seong Jae Hwang}{yyy}
\end{icmlauthorlist}

\icmlaffiliation{yyy}{Department of Artificial Intelligence, Yonsei University, Seoul, Republic of Korea}

  \icmlcorrespondingauthor{Seong Jae Hwang}{seongjae@yonsei.ac.kr}

  \icmlkeywords{Inductive Bias, Vision Models}

  \vskip 0.3in
]



\printAffiliationsAndNotice{}  

\begin{abstract}
The standard Transformer architecture relies on a rigid pattern that alternates Attention and Feed-Forward Network (FFN) layers. Despite its widespread adoption, the inductive bias imposed by this strict separation has not been systematically examined. In this work, we investigate the necessity of the Attention--FFN dichotomy in Vision Transformers (ViTs). To facilitate this analysis, we introduce the AttenFeed module, a unified component that integrates the functional properties of both Attention and FFN. Based on this module, we devise the unified Vision Transformer (uViT), which replaces the conventional alternating Attention--FFN structure with a sequence of AttenFeed modules. We then use uViT as a control group that relaxes the Attention--FFN dichotomy of the standard ViT and systematically compare the two models across multiple datasets and model scales. Our experiments reveal that the Attention--FFN dichotomy can hinder performance at smaller model scales due to the rigid parameter allocation of ViTs. The AttenFeed module and uViT serve as new analytical tools for understanding the Attention--FFN structure and offer theoretical insights into the heuristically designed architecture of conventional ViTs.
\end{abstract}

\section{Introduction}
\label{sec:motivation}

The Transformer architecture~\cite{vaswani} has demonstrated exceptional performance across diverse domains, establishing itself as a universal architecture~\cite{bert, gpt, vit, clip, dit, wav2vec, alphafold}. Although originally developed for machine translation, it is highly encouraging that the Transformer exhibits robust performance even in disparate fields. In particular, Vision Transformer (ViT)~\cite{vit} has achieved results surpassing those of Convolutional Neural Networks (CNNs)~\cite{lenet, alexnet, resnet}, which had long served as the de facto standard in the computer vision domain.

This success of ViT is largely attributed to the attention mechanism. Unlike CNNs, which structurally enforce a locality inductive bias~\cite{cnn-bias1, cnn-bias2, deeplearning}, attention-based models are required to learn this property directly from data.
While this absence of structural constraints poses an additional learning challenge, the lack of bias grants the model the freedom to learn superior mechanisms from scratch.
Consequently, ViTs not only outperform CNNs in capturing long-range dependencies~\cite{convit, variablecreation} but also exhibit advantageous emergent properties such as register tokens~\cite{regtoken} that are absent in traditional CNNs.

As inductive bias has played a pivotal role in the advancement of Transformers, some research have been dedicated to understanding the biases intrinsic to the architecture. Briefly, it has been observed that Transformers exhibit a structural bias toward symmetric functions~\cite{transformer-bias1} and tend to concentrate attention on a sparse subset of tokens~\cite{transformer-bias2}.
Despite these findings, the inductive bias imposed by the most fundamental structural choice of the Transformer, alternation of Attention and Feed-Forward Network (FFN), remains largely unexplored.

Moreover, recent studies on Mixture-of-Experts (MoE) provide evidence that this ``baked-in'' alternating structure of Attention and FFN may not be indispensable~\cite{diffskip, skipgpt, router-tuning, whatmatters}. For instance, SkipGPT~\cite{skipgpt} trained a router to conditionally drop attention or FFN blocks, revealing that consecutive interactions (e.g., Attention-Attention or FFN-FFN) are occasionally more effective than the standard alternating pattern. These studies motivate an investigation into when the Attention--FFN structure in Transformers is necessary and when it may be redundant.

Given these observations, understanding the inductive bias imposed by the repeating Attention--FFN structure of the Transformer becomes crucial. Therefore, we aim to analyze this ``Attention--FFN dichotomy'' in ViT. To achieve this, (i) we formulate the AttenFeed module, a single component capable of exhibiting the properties of both Attention and FFN. (ii) Then we build the unified Vision Transformer (uViT) composed of a sequence of AttenFeed modules, thereby removing the strict Attention--FFN dichotomy. (iii) Lastly, we compare uViT with ViT to analyze under what circumstances and to what extent the conventional Attention--FFN dichotomy facilitates the learning process.

To formulate the AttenFeed module, we start from the interpretation of UMoE~\cite{umoe} that views the standard Attention as an FFN with no activation function. By incorporating a non-linear activation function into standard attention, we enable its interpretation as a conventional FFN. In \cref{exp:math_description}, we mathematically demonstrate that the AttenFeed module is a generalized form of both Attention and FFN. Then, in \cref{exp:attention,exp:ffn,exp:weight_comparison}, we experimentally show that the AttenFeed module simultaneously possesses the properties of both Attention and FFN. These results confirm that AttenFeed relaxes the Attention--FFN dichotomy, subsuming the functionalities of both.

\begin{figure}[t]
  \centering
  \includegraphics[width=\columnwidth]{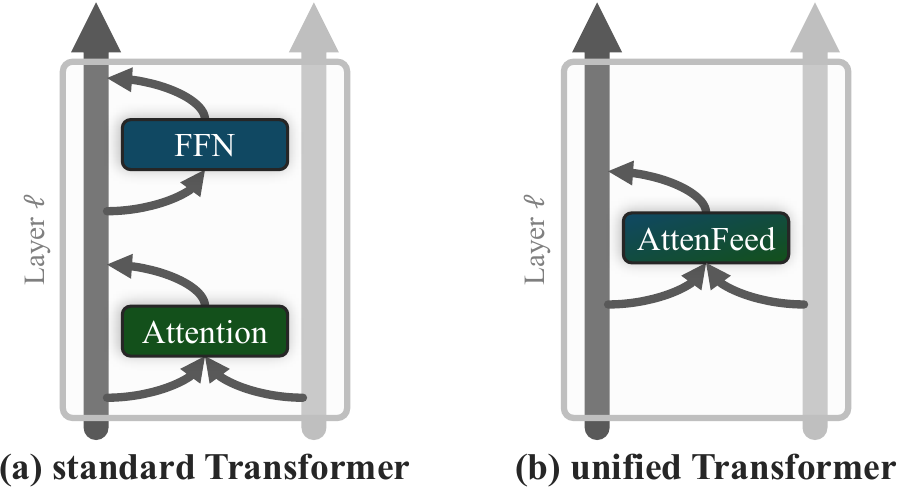}
  \caption{\textbf{Architecture illustrations}. (a) shows a layer of the standard Transformer, and (b) shows a layer of the unified Transformer, a control architecture for studying the Attention--FFN dichotomy.}
  \label{fig:teaser}
  \vspace{-10pt}
\end{figure}

Next, we design uViT, an architecture composed of a sequence of AttenFeed modules (Figure~\ref{fig:teaser}b). Since the AttenFeed module relaxes the structural inductive bias imposed by the Attention--FFN separation in conventional ViTs, we hypothesize that uViT exhibits weaker inductive bias than standard ViTs. We empirically validate this conjecture in \cref{exp:transfer}. Consequently, uViT serves as a control group for investigating the Attention--FFN dichotomy of ViTs, as it subsumes the functional properties of both Attention and FFN while relaxing the strict structural constraints.

Finally, we pretrain both ViT and uViT on various datasets and compare their performance in \cref{exp:scaling}. The results show that while uViT significantly outperforms ViT at smaller model scales, the performance gap between the two architectures narrows and becomes marginal as the model capacity increases. The experiments in \cref{exp:inductive} confirm that the inductive bias imposed by the Attention--FFN dichotomy in standard ViTs actually hinders performance at smaller scales.

The significance of our proposed AttenFeed module extends beyond the mere integration of standard Attention and FFN. Rather, it aims to open a new avenue for theoretically and empirically investigating the heuristically established alternating Attention--FFN structure of Transformers~\cite{attentiononly, empirical}. We anticipate that the AttenFeed module and uViT will facilitate a more systematic investigation into the architectural inductive biases that have long underpinned the success of ViT.

Our main contributions are summarized as follows:
\begin{itemize}
    \item[$\circ$] \textbf{Formulation of the AttenFeed Module.} We propose the AttenFeed module, a unified component mathematically and empirically proven to generalize the properties of both standard Attention and FFNs.
    \item[$\circ$] \textbf{Introduction of uViT.} We design the unified Vision Transformer (uViT) composed entirely of AttenFeed modules, effectively removing the strict Attention--FFN dichotomy of standard Transformers.
    \item[$\circ$] \textbf{Analysis of Inductive Bias at Scale.} Through experiments using uViT, we reveal that the conventional Attention--FFN dichotomy acts as a restrictive inductive bias that impedes effective learning at smaller model scales.
\end{itemize}

\section{Related Work}
\label{sec:relwork}

Inductive bias serves as a fundamental architectural constraint that guides a model toward a specific solution space. While these biases often accelerate early-stage training, they can arguably act as a bottleneck, limiting the performance ceiling when data becomes sufficiently large~\cite{scalingmlps, vit, generalization, bitterlesson}. For instance, the Vision Transformer (ViT)~\cite{vit} relaxes the image-specific inductive biases inherent in Convolutional Neural Networks (CNNs), such as translation equivariance and locality~\cite{cnn-bias1, cnn-bias2, deeplearning}.

As inductive bias has emerged as a key concept in understanding Transformer architectures, a growing body of literature has sought to characterize the specific biases that underpin Transformer performance. Transformers exhibit a structural bias toward symmetric functions~\cite{transformer-bias1} and tend to concentrate attention on a sparse subset of tokens~\cite{transformer-bias2}. From a task-specific perspective, the pairwise computation of self-attention inherently favors tasks requiring relational reasoning~\cite{transformer-bias3}. Furthermore, compared to State-Space Models (SSMs)~\cite{gu2024mamba}, the structural bias of Transformers makes them more suitable for the copy task~\cite{transformer-bias4}.

Building upon these insights, subsequent research has focused on actively manipulating these biases to enhance training efficiency and performance~\cite{vitae, earlyconv, convit}. For instance, Swin Transformer~\cite{swintransformer} reintroduces locality and translation invariance through a hierarchical architecture and shifted window attention mechanism. Conversely, MLP-Mixer~\cite{mlpmixer} and ResMLP~\cite{resmlp} minimize the inductive bias of token mixing by replacing self-attention with vanilla MLPs, demonstrating superior performance over standard Transformers on larger datasets. Beyond structural modifications, several studies~\cite{spatialloss, ditinductive} propose methods for injecting inductive biases during the training phase to enable more efficient training.

Despite these advancements, to the best of our knowledge, there is no research exploring the necessity of the rigid functional separation between Attention and FFN. While GAU~\cite{gau} attempts to merge GLU and Attention, it was primarily designed to compensate for the performance limitations of linear attention rather than to minimize inductive bias. We propose a unified module that generalizes both components, thereby enabling, for the first time, a systematic investigation of the Attention--FFN dichotomy.
\section{Method}
The primary objective of this paper is to investigate when the inductive bias imposed by Attention--FFN structure becomes beneficial. To facilitate this investigation, we devise an architecture designed to relax this structural separation. Specifically, in \cref{sec:prelim}, we briefly review the conventional Vision Transformer (ViT) block. Next, in \cref{sec:method_af}, we introduce the AttenFeed module, a component that conceptually unifies Attention and FFN. Finally, in \cref{sec:method_uvit}, we present the overall unified Vision Transformer (uViT) architecture, designed to relax the explicit decoupling of Attention and FFN in ViT.

\subsection{Preliminary}
\label{sec:prelim}

A conventional ViT takes a \texttt{[CLS]} token and {\small $T$} image patch tokens as input. Let {\small $\mathbf{X}^{\ell} \in \mathbb{R}^{(1+T) \times d}$} denote the input token sequence for layer {\small $\ell \in \{1,\ldots,L\}$}, and {\small $\mathbf{x}_i^{\ell} \in \mathbb{R}^d$} represent the $i^{\text{th}}$ token in the sequence. A standard ViT block (Figure~\ref{fig:teaser}a) comprises an alternating structure of multi-head Attention and FFNs, formulated as follows~\cite{vit, primer}:
\begin{equation}
{\small
    \begin{aligned}
        \mathbf{x}_i^{\text{mid}, \ell} &= \mathbf{x}_i^{\ell-1} + \sum_h \text{Attn}^{\ell,h}(\text{LN}_1^\ell(\mathbf{X}^{\ell-1}))_i, \\
        \mathbf{x}_i^{\ell} &= \mathbf{x}_i^{\text{mid}, \ell} + \text{FFN}^\ell (\text{LN}_2^\ell(\mathbf{x}_i^{\text{mid}, \ell})),
    \end{aligned}
}
\end{equation}
where {\small $\text{LN}_1^\ell$} and {\small $\text{LN}_2^\ell$} are layer normalizations.

The attention head in a Transformer is typically formulated as\footnote{For notational simplicity, we omit bias terms throughout this paper. This comes without loss of generality, since biases can be absorbed into the weight matrices by appending a constant 1 to the input vector.}:
\begin{equation}
\label{eqn:attention}
{\small
    \begin{aligned}
        \text{Attn}^{\ell, h}(\mathbf{X}^{\ell-1})_i = \sum_{j} a_{i,j}^{\ell, h} \mathbf{x}_j^{\ell -1} \mathbf{W}_V^{\ell, h} \mathbf{W}_O^{\ell, h},
    \end{aligned}
}
\end{equation}
where {\small $\mathbf{W}_V^{\ell, h} \in \mathbb{R}^{d \times d_h}$} and {\small $\mathbf{W}_O^{\ell, h} \in \mathbb{R}^{d_h \times d}$}, and {\small $h \in \{1,\ldots,H\}$} denotes the head index. Following the terminology of \cite{mathematical}, we refer to the computation of the attention weights $a_{i,j}^{\ell, h}$ as the QK circuit, and the transformation {\small $\mathbf{x}_j^{\ell -1} \mathbf{W}_V^{\ell, h} \mathbf{W}_O^{\ell, h}$} as the OV circuit\footnote{While \cite{mathematical} originally includes the unembedding matrix in the OV circuit definition, we exclude it in this work to focus on layer-wise operations.}.

The Attention layer is followed by a position-wise FFN:
\begin{equation}
\label{eqn:ffn}
{\small
    \begin{aligned}
        \text{FFN}^{\ell}(\mathbf{x}^{\text{mid},\ell}_i) = \sigma(\mathbf{x}^{\text{mid},\ell}_i \mathbf{W}_\text{in}^{\ell}) \mathbf{W}_\text{out}^{\ell},
    \end{aligned}
}
\end{equation}
where {\small $\sigma$} is a non-linear activation function, and {\small $\mathbf{W}_\text{in}^{\ell} \in \mathbb{R}^{d \times d_\text{FFN}}$}, {\small $\mathbf{W}_\text{out}^{\ell} \in \mathbb{R}^{d_\text{FFN} \times d}$}.

\subsection{AttenFeed Module}
\label{sec:method_af}

UMoE~\cite{umoe} interprets the OV circuit in \cref{eqn:attention} as an FFN without an activation function. Inspired by this interpretation, we apply minimal modifications to the standard attention mechanism to endow it with FFN-like behavior. Specifically, we simply introduce a non-linear activation function {\small $\sigma$} into the value projection to formulate the AttenFeed module as follows:
\begin{equation}
\label{eqn:af}
{\small
    \begin{aligned}
        \text{AttenFeed}^{\ell, h}(\mathbf{X}^{\ell-1})_i \coloneq \sum_{j} a_{i,j}^{\ell, h} \sigma(\mathbf{x}_j^{\ell -1} \mathbf{W}_V^{\ell, h}) \mathbf{W}_O^{\ell, h}.
    \end{aligned}
}
\vspace{-5pt}
\end{equation}
Unless otherwise mentioned, we use GELU~\cite{gelu} for {\small $\sigma$}.

This modification now allows the OV circuit to be interpreted as a standard FFN by explicitly incorporating a non-linear activation. At the same time, as this requires minimal alterations to the conventional attention mechanism, we expect the module to retain the intrinsic properties of the standard attention. In other words, while {\small $\mathbf{W}_V$} and {\small $\mathbf{W}_O$} can be interpreted as the value and output matrices of standard attention, respectively, the module can also be viewed as a standard FFN by redefining them as {\small $\mathbf{W}_V \coloneqq \mathbf{W}_\text{in}$} and {\small $\mathbf{W}_O \coloneqq \mathbf{W}_\text{out}$} given in \cref{eqn:ffn}. We show in \cref{exp:math_description} that the AttenFeed module serves as a mathematical generalization of Attention and FFN, and further empirically validate this finding in \cref{exp:attention,exp:ffn,exp:weight_comparison}.

\subsection{Unified Vision Transformer}
\label{sec:method_uvit}

Based on the AttenFeed module, we propose the unified Vision Transformer (uViT) (Figure~\ref{fig:teaser}a), where the update rule is given by:
\begin{equation}
\label{eqn:utransformer}
{\small
    \begin{aligned}
        \mathbf{x}_i^{\ell} &= \mathbf{x}_i^{\ell - 1} + \sum_{h}\text{AttenFeed}^{\ell,h} (\text{LN}^\ell(\mathbf{X}^{\ell - 1}))_i.
    \end{aligned}
}
\end{equation}
Thus, uViT replaces the coupled Attention--FFN structure of the standard Transformer with a single AttenFeed module per layer. We argue that \textit{this design alleviates structural constraints and the associated inductive bias of the standard Transformer} in \cref{exp:transfer}.

Structurally, uViT is equivalent to removing the FFN from a standard Transformer and introducing a parameter-free non-linear activation function, which results in approximately three times fewer parameters~\cite{ffn1}. To ensure a fair comparison, we align the total parameter count of uViT with that of the standard baseline by proportionally increasing its number of layers and model dimension. The detailed model configurations are summarized in \cref{tab:model_config}.

\begin{table}[tb]
  \caption{\textbf{Configurations of the ViT and uViT models.} $L, H$, and $D$ denote the number of layers, heads, and model dimension, respectively.
  }
  \label{tab:model_config}
  \centering
  \setlength{\tabcolsep}{11pt}
  \begin{tabular}{rccc}
    \toprule
    Name & Params. & ($L$, $H$, $D$) & MLPs \\
    \midrule
    ViT-S  & 21M & (12, 6, 384) & $\checkmark$ \\
    uViT-S & 20M & (12, 8, 640) & $\times$ \\
    \midrule
    ViT-B  & 86M & (12, 12, 768) & $\checkmark$ \\
    uViT-B & 80M & (12, 16, 1280) & $\times$ \\
    \midrule
    ViT-L  & 303M & (24, 16, 1024) & $\checkmark$ \\
    uViT-L & 289M & (28, 20, 1600) & $\times$ \\
  \bottomrule
  \end{tabular}
\end{table}

\section{AttenFeed as a Generalization of Attention and FFN}
\label{sec:integrated}
In \cref{exp:math_description}, we mathematically demonstrate that the AttenFeed module generalizes both Attention and FFN. Furthermore, in \cref{exp:ffn} and \cref{exp:attention}, we empirically show that the AttenFeed module exhibits the distinct characteristics of both Attention and FFN, respectively. Finally, in \cref{exp:weight_comparison}, we reveal that the parameters {\small $\mathbf{W}_V$} and {\small $\mathbf{W}_O$} within the AttenFeed module converge to intermediate values between those of standard attention and standard FFN. Ultimately, these findings indicate that the \textit{AttenFeed module mitigates the inductive biases inherent in Attention--FFN separation, effectively inheriting the properties of both.}

\subsection{Mathematical Description}
To claim that the AttenFeed module generalizes both Attention and FFN, we present the following propositions. The detailed proofs are provided in \cref{appendix:math_description}.
\label{exp:math_description}
\begin{proposition}
\label{prop1}
If an AttenFeed head {\small $\text{AttenFeed}^{\ell, h}$} is self-focused, then the head's output approximates a position-wise FFN: {\small $\text{AttenFeed}^{\ell, h}(\mathbf{X}^{\ell-1})_i \approx \text{FFN}(\mathbf{x}_i^{\ell-1})$}.
\end{proposition}
\begin{proposition}
\label{prop2}
Assume the pre-activation inputs to the GELU function are sufficiently small, such that {\small $\| \mathbf{x}_j^{\ell-1} \mathbf{W}_V^{\ell, h} \| \to 0$} for all {\small $j$}. Under this condition, the AttenFeed head approximates a standard attention head: {\small $\text{AttenFeed}^{\ell, h}(\mathbf{X}^{\ell-1}) \approx \text{Attn}^{\ell,h}(\mathbf{X}^{\ell-1})$}.
\end{proposition}
\begin{proposition}
\label{prop3}
Let an AttenFeed head $\text{AttenFeed}^{\ell, h}$ follow the conditions of Proposition~\ref{prop1}. Then, there exists a set of parameterizations for the AttenFeed head such that {\small $\text{AttenFeed}^{\ell,h}(\mathbf{X}^{\ell-1}) \approx \text{Attn}^{\ell,h}(\mathbf{X}^{\ell-1})$}.
\end{proposition}

Proposition~\ref{prop1} implies that when the attention of the AttenFeed module is self-focused, it becomes equivalent to an FFN. Proposition~\ref{prop2} demonstrates that the AttenFeed module can approximate standard attention. Proposition~\ref{prop3} more rigorously shows that when the AttenFeed module is self-focused, it approximates standard attention under looser conditions than those required in Proposition~\ref{prop2}. Consequently, these propositions mathematically prove that the AttenFeed module is a generalization of both Attention and FFN.

\subsection{FFN-like Behavior of the AttenFeed Module}
\label{exp:ffn}
\begin{table*}[tb]
  \caption{\textbf{Performance and representation rank comparison on CIFAR10 and CIFAR100.} The results demonstrate that uViT (composed of AttenFeed modules) consistently maintains a higher normalized rank and achieves superior accuracy compared to xViT (ViT without FFNs), confirming that the AttenFeed module effectively inherits the rank-preserving properties of FFNs.}
  \label{tab:ffn}
  \centering
  \setlength{\tabcolsep}{5.6pt}
  \begin{tabular}{ccccccccc}
    \toprule
    & & & \multicolumn{3}{c}{CIFAR10} & \multicolumn{3}{c}{CIFAR100} \\
    \cmidrule(lr){4-6} \cmidrule(lr){7-9}
    Model & ($L$, $H$, $D$) & Params. & Test acc. & $\text{Rank}_\text{(pre-head)}$ & $\text{Rank}_\text{(post-head)}$ & Test acc. & $\text{Rank}_\text{(pre-head)}$ & $\text{Rank}_\text{(post-head)}$ \\
    \midrule
        xViT & \multirow{2}{*}{(6, 4, 128)} & \multirow{2}{*}{0.43M} & 78.34\% & 0.375 & 0.70 & 50.94\% & 0.406 & 0.25 \\
        uViT & & & 84.00\% & 0.383 & 0.70 & 58.56\% & 0.508 & 0.28 \\
        \cmidrule(lr){1-3} \cmidrule(lr){4-6} \cmidrule(lr){7-9} 
        xViT & \multirow{2}{*}{(6, 4, 160)} & \multirow{2}{*}{0.67M} & 78.59\% & 0.356 & 0.70 & 52.73\% & 0.463 & 0.29 \\
        uViT & & & 85.16\% & 0.369 & 0.70 & 59.80\% & 0.506 & 0.32 \\
        \cmidrule(lr){1-3} \cmidrule(lr){4-6} \cmidrule(lr){7-9} 
        xViT & \multirow{2}{*}{(12, 4, 128)} & \multirow{2}{*}{0.83M} & 80.13\% & 0.430 & 0.70 & 52.82\% & 0.586 & 0.29 \\
        uViT & & & 85.60\% & 0.445 & 0.80 & 58.69\% & 0.586 & 0.30 \\
    \bottomrule
  \end{tabular}
\end{table*}

To empirically validate Proposition~\ref{prop1}, we examine whether the AttenFeed module exhibits the functional properties of an FFN. Specifically, we focus on the established observation that FFNs play a critical role in mitigating rank collapse within Transformers~\cite{ffn_collapse, ffn_collapse2, ffn_collapse3, park2022vision}. Rank collapse refers to a phenomenon where the latent representations of different tokens become increasingly correlated as the network depth increases. This leads to “token uniformity,” where the model loses its expressivity because the features of distinct tokens become nearly indistinguishable. Models composed solely of attention layers are therefore known to exhibit a reduction in the rank of their representations~\cite{ffn_collapse}.

As shown in \cref{eqn:af}, the AttenFeed module introduces a GELU activation into the attention mechanism. This modification endows the conventional attention mechanism with FFN-like properties (\cref{exp:math_description}). Thus, to empirically verify that AttenFeed module possesses FFN-like behavior, we compare the representational rank of uViT against a ViT variant in which the FFN is removed, resulting in a model composed solely of attention layers. We denote this FFN-removed ViT as xViT. Given that FFNs are known to play a pivotal role in preventing rank collapse, we hypothesize that xViT will exhibit a relatively lower rank. Conversely, we expect uViT to maintain a higher rank, as its AttenFeed modules are designed to integrate FFN-like functionalities.

\paragraph{Settings.}
First, we train xViT on a training set. Then, we freeze the model and record all the outputs of the \texttt{[CLS]} token on the associated validation set. Next, we measure the rank by performing Singular Value Decomposition (SVD) and counting the number of eigenvalues that account for 95\% of the variance~\cite{rankmeasure1, rankmeasure2} (for details, see \cref{appendix:rank_procedure}). This rank value is then divided by the maximum possible rank to present the normalized rank~\cite{normalizedrank}.\footnote{For instance, given {\small $N$} \texttt{[CLS]} token representations in {\small $\mathbb{R}^{N \times d}$}, the maximum possible rank is {\small $d$}; thus, if the measured rank is {\small $r$}, the normalized rank is computed as {\small $r / d$}.} Subsequently, we insert a GELU activation into the attention module of xViT to transform it into uViT, and repeat the identical procedure to measure its normalized rank. We use CIFAR-10 and CIFAR-100~\cite{cifar} as train datasets.

\paragraph{Results.}
The experimental results are summarized in \cref{tab:ffn}. We define pre-head rank and post-head rank as the normalized ranks of representations immediately preceding and following the classification head. When comparing xViT and uViT under identical configurations, uViT consistently exhibits a higher output rank, which correlates with improved test accuracy. These trends remain robust across various datasets (\cref{appendix:exp_result_detail}). Given that the only structural disparity between uViT and xViT is the presence of the GELU activation, these findings suggest that the non-linear function imparts FFN-like, rank-preserving properties to the AttenFeed module.

\subsection{Attention-like Behavior of AttenFeed Module}

\label{exp:attention}
To empirically validate Proposition~\ref{prop2} and \ref{prop3}, we examine whether the AttenFeed module exhibits attention-like behavior through attention-based segmentation. Attention maps in ViTs visualize the salient image regions that the model focuses on when making predictions~\cite{vit, tscam, regtoken}. These maps can therefore be leveraged to obtain segmentation masks for object regions~\cite{conceptattention, chefer2021transformer}. Thus, we assess whether the attention maps produced by ViT can also be utilized to generate meaningful segmentation masks.

\paragraph{Settings.}
First, we pretrain ViT-S/B/L and uViT-S/B/L on the ImageNet-1k dataset~\cite{imagenet}. To ensure a fair comparison, all models are trained under identical conditions (please refer to \cref{appendix:exp_setting_detail} for further details). After pretraining, we freeze the models and perform image segmentation. Following \cite{conceptattention, chefer2021transformer}, we extract the attention map of the \texttt{[CLS]} token from the final layer for a given input image. Regions with values exceeding the mean attention score are binarized to produce the segmentation mask.
We evaluate the segmentation performance of these attention-based masks on ImageNet-Segmentation~\cite{imagenetsegdataset} and PASCAL-VOC~\cite{pascalvoc} datasets.

\paragraph{Results.}
\begin{table}[tb]
  \caption{\textbf{Segmentation performance using the attention maps from ViT and uViT.} The results demonstrate that the AttenFeed module effectively generates meaningful segmentation masks, performing comparably to standard attention modules.}
  \label{tab:attn}
  \centering
  \setlength{\tabcolsep}{4.6pt}
  \begin{tabular}{rcccccc}
    \toprule
    & \multicolumn{3}{c}{ImageNet-Seg.} & \multicolumn{3}{c}{Pascal-VOC {\tiny (Single class)}} \\ \cmidrule(lr){2-4}
    \cmidrule(lr){5-7}
    Model & Acc. & mIoU & mAP & Acc. & mIoU & mAP \\
    \midrule
    ViT-S & 0.772 & 0.629 & 0.821 & 0.745 & 0.550 & 0.853 \\
    uViT-S & 0.699 & 0.538 & 0.806 & 0.689 & 0.490 & 0.847 \\
    \midrule
    ViT-B & 0.725 & 0.569 & 0.801 & 0.707 & 0.480 & 0.844 \\
    uViT-B & 0.680 & 0.516 & 0.775 & 0.682 & 0.449 & 0.828 \\
    \midrule
    ViT-L & 0.666 & 0.500 & 0.765 & 0.677 & 0.422 & 0.819 \\
    uViT-L & 0.617 & 0.447 & 0.767 & 0.631 & 0.384 & 0.821 \\
    \bottomrule
  \end{tabular}
\end{table}

\begin{figure*}[t]
  \centering
  \includegraphics[width=0.95\textwidth]{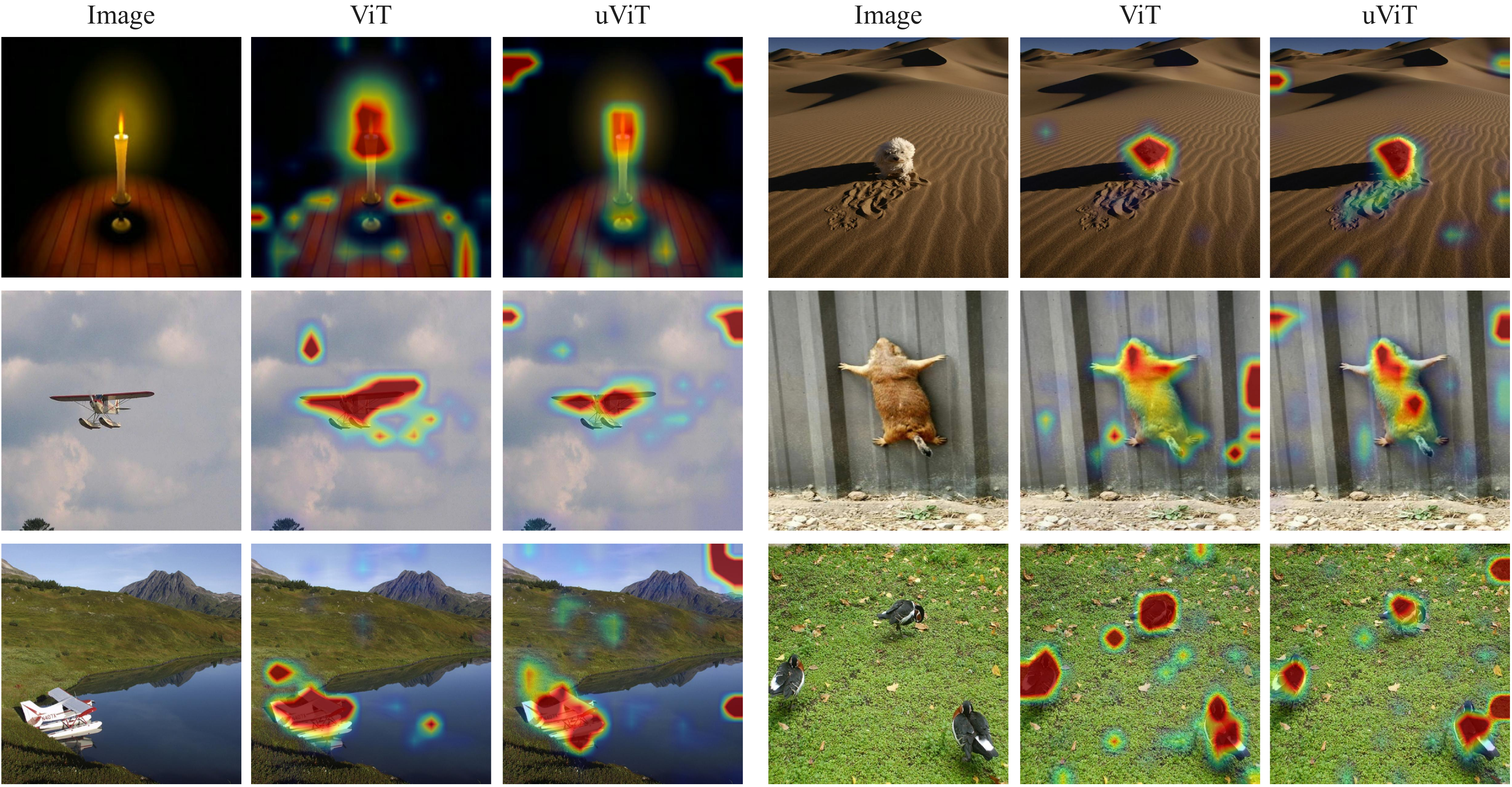}
  \caption{\textbf{Qualitative comparison of attention maps between ViT and uViT.}}
  \label{fig:attn_qualitative}
\end{figure*}

The results are presented in \cref{tab:attn}. The attention map segmentation performance of uViT is slightly lower than that of ViT. This is expected, since the AttenFeed module in uViT must simultaneously act as both attention and FFN (\cref{exp:ffn}), unlike the attention module in ViT. Nevertheless, the performance gap is not substantial, with uViT-L even achieving a higher mAP than ViT-L.

Furthermore, we present qualitative segmentation results in \cref{fig:attn_qualitative}. The first column shows the original input images, whereas the second and third columns display the attention maps of ViT and uViT, respectively. As shown in the figure, uViT effectively captures salient object regions in the input image, suggesting that the AttenFeed module exhibits attention-like properties. Additional qualitative results are provided in \cref{appendix:qualitative}.

\subsection{Comparing Parametric Spaces of uViT and ViT}
\label{exp:weight_comparison}
As discussed in \cref{sec:method_af} and \cref{appendix:math_description}, {\small $\mathbf{W}_V$} in the AttenFeed module can be interpreted as either {\small $\mathbf{W}_V$} in the attention mechanism of a standard transformer or as {\small $\mathbf{W}_\text{in}$} in the FFN. Similarly, {\small $\mathbf{W}_O$} in the AttenFeed module can be interpreted as either {\small $\mathbf{W}_O$} or {\small $\mathbf{W}_\text{out}$} in the standard transformer.
Therefore, if our assumption that the AttenFeed module generalizes both Attention and FFN is valid, the parameters {\small $\mathbf{W}_V$} and {\small $\mathbf{W}_O$} of the AttenFeed module are expected to lie in an intermediate region between the corresponding parameters of the Attention and FFN modules in a standard Transformer. In this section, we empirically verify this hypothesis.

\paragraph{Settings.}
We visualize the parameter distributions of the ViT-B and uViT-B models used in \cref{exp:attention}. The analysis is conducted from two perspectives:
\begin{itemize}
\item[$\circ$] Visualizing the value distributions of {\small $\mathbf{W}_V$} and {\small $\mathbf{W}_\text{in}$} in ViT, and {\small $\mathbf{W}_V$} in uViT, at a specific layer (\cref{fig:weight_comparison}a).
\item[$\circ$] Visualizing the value distributions of {\small $\mathbf{W}_O$} and {\small $\mathbf{W}_\text{out}$} in ViT, and {\small $\mathbf{W}_O$} in uViT, at a specific layer (\cref{fig:weight_comparison}b).
\end{itemize}
The distribution of each matrix element is represented by connecting the midpoints of histogram bins to form a continuous density profile.

\paragraph{Results.}
\begin{figure*}[t]
  \centering
  \includegraphics[width=0.95\textwidth]{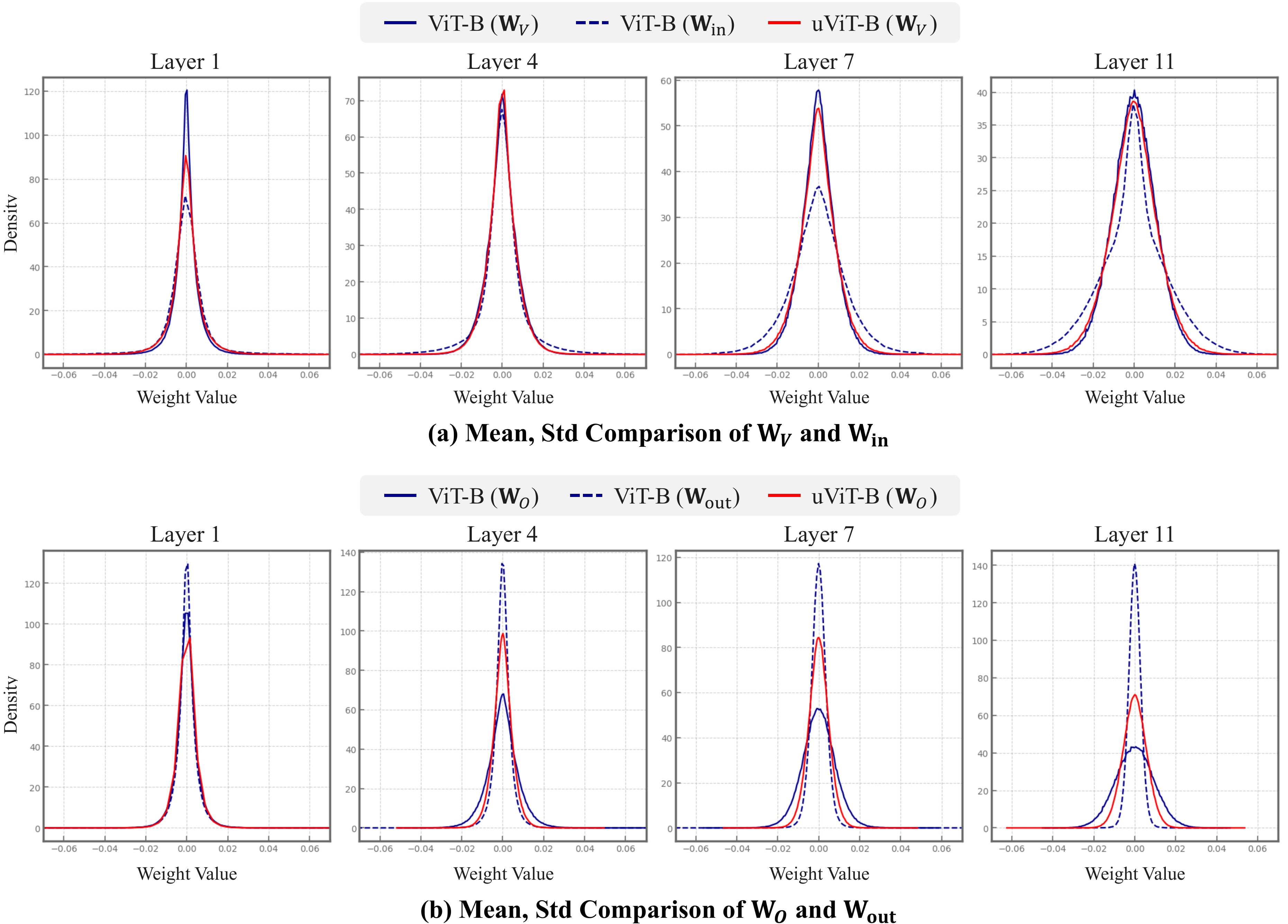}
  \caption{\textbf{Weight value distributions of uViT and standard ViT.} The results show that uViT parameters occupy a transitional space between standard attention and FFN regimes, providing empirical evidence for the unified nature of the AttenFeed module.}
  \label{fig:weight_comparison}
\end{figure*}
As illustrated in \cref{fig:weight_comparison}, the values of {\small $\mathbf{W}_V$} and {\small $\mathbf{W}_O$} in the AttenFeed module occupy an intermediate range between those of the standard attention and FFN components. These results suggest that the parameter space of the AttenFeed module does not deviate significantly from the regimes of standard attention and FFN, supporting the functional generalization of the AttenFeed module. Additional results are provided in \cref{appendix:exp_result_detail}.

\begin{center}
\begin{tcolorbox}[
    width=0.48\textwidth,
    colframe=darkgray,
    colback=orange!3,
    boxrule=1pt,
    arc=1mm,
    boxsep=2pt
]
\textbf{Finding of \cref{sec:integrated}}:
AttenFeed module mathematically generalizes standard Attention and FFN, and empirically exhibits properties of both attention and FFN.
\end{tcolorbox}
\end{center}
\section{Impact of Inductive Bias in ViT}
\label{exp:inductive}
Unified Vision Transformer (uViT) consists of a sequence of AttenFeed modules, each capable of functioning as both Attention and FFN. Consequently, unlike the standard ViT, uViT is expected to be liberated from the inductive biases inherent in the conventional Attention--FFN dichotomy. In \cref{exp:transfer}, we experimentally demonstrate that uViT indeed possesses weaker inductive bias than ViT. Subsequently, in \cref{exp:scaling}, we compare the pretraining performance of uViT and ViT, showing that the Attention--FFN dichotomy in ViT has a detrimental effect on performance when the model size is small.

\subsection{Assessing Inductive Bias via Transfer Learning}
\label{exp:transfer}
\begin{table*}[tb]
  \caption{\textbf{Transfer learning performance comparison between ViT and uViT.} The results include Top-1 accuracy (\%) for linear probing and full fine-tuning across four downstream datasets (SVHN, CIFAR10, CIFAR100, and STL10).}
  \label{tab:transfer}
  \centering
  \setlength{\tabcolsep}{10pt}
  \begin{tabular}{rcccccccc}
    \toprule
    & \multicolumn{4}{c}{Linear Probing (\%)} & \multicolumn{4}{c}{Full Fine-tuning (\%)} \\
    \cmidrule(lr){2-5}
    \cmidrule(lr){6-9}
    Model & SVHN & CIFAR10 & CIFAR100 & STL10 & SVHN & CIFAR10 & CIFAR100 & STL10 \\
    \midrule
    ViT-S & 40.03 & 89.94 & 80.84 & 58.64 & 93.03 & 75.68 & 51.09 & 52.13 \\
    uViT-S & 56.57 & 95.40 & 90.99 & 72.67 & 91.76 & 73.77 & 47.77 & 50.63 \\
    \midrule
    ViT-B & 47.91 & 91.73 & 73.10 & 96.93 & 91.08 & 73.58 & 44.73 & 50.04 \\
    uViT-B & 60.84 & 93.75 & 77.85 & 97.36 & 91.70 & 74.98 & 47.41 & 51.56 \\
    \midrule
    ViT-L & 50.69 & 93.23 & 76.05 & 97.63 & 87.35 & 70.10 & 43.90 & 47.66 \\
    uViT-L & 58.80 & 94.41 & 79.16 & 97.54 & 93.33 & 74.28 & 46.20 & 53.35 \\
    \bottomrule
  \end{tabular}
\end{table*}

uViT relaxes the strict separaration of Attention and FFN architecture in ViT. As a result, uViT is expected to exhibit less inductive bias than ViT and therefore holds the potential to learn more generalized representations~\cite{vit}. To examine this hypothesis, we compare the transfer learning performance of pre-trained uViT and ViT on downstream datasets. If uViT indeed possesses less inductive bias than ViT, it should demonstrate better transfer performance than ViT.

\paragraph{Settings.}
We further train ViT-S/B/L and uViT-S/B/L, which were pre-trained on ImageNet-1k as described in \cref{exp:attention}, on the SVHN, CIFAR10, CIFAR100, and STL10 datasets~\cite{svhn, cifar, stl10}. For a fair comparison, all training hyperparameters were kept strictly identical (see \cref{appendix:exp_setting_detail}). To evaluate the flexibility of the output representations of pre-trained models, we report the performance of training only the classification head (linear probing). In addition, to assess the flexibility of the model’s internal parameters, we present full fine-tuning results.

\paragraph{Results.}
The experimental results are presented in \cref{tab:transfer}. uViT outperformed ViT across all evaluated scenarios for linear probing, and for full fine-tuning, uViT achieved superior performance in the majority of cases. These results indicate that both the output representations and the internal parameters of uViT possess greater flexibility when adapting to unseen datasets. These results provide strong evidence that uViT exhibits less inductive bias than ViT.

\begin{center}
\begin{tcolorbox}[
    width=0.48\textwidth,
    colframe=darkgray,
    colback=orange!3,
    boxrule=1pt,
    arc=1mm,
    boxsep=2pt
]
\textbf{Finding of \cref{exp:transfer}}: uViT exhibits less inductive bias imposed by the Attention--FFN dichotomy than ViT.
\end{tcolorbox}
\end{center}

\subsection{The Role of Inductive Bias at Scale}
\label{exp:scaling}
\begin{table}[tb]
  \caption{\textbf{Pre-training performance across different model scales.} uViT shows a significant performance advantage over ViT at the Small scale, while the gap narrows at the Base scale. {\small $^\dagger$}We doubled the batch size for the experiments, since convergence of uViT was unstable under the default setting.
  }
  \label{tab:pretrain}
  \centering
  \setlength{\tabcolsep}{8pt}
  \begin{tabular}{rccc}
    \toprule
    & \multicolumn{3}{c}{Test Acc. (\%)} \\
    \cmidrule(lr){2-4} Model & ImageNet-1k & Places365 & iNaturalist \\
    \midrule
        ViT-S & 53.64 & 44.63 & 52.82 \\
        uViT-S & 67.11 & 50.25 & 66.81 \\
    \midrule
        ViT-B & 74.22 & 53.37 & 76.28 \\
        uViT-B & 74.90 & 54.45 & 78.26 \\
    \midrule
        ViT-L & 77.33 & 54.75 & 85.39$^\dagger$ \\
        uViT-L & 77.18 & 55.14 & 84.49$^\dagger$ \\
    \bottomrule
  \end{tabular}
\end{table}

\cref{sec:integrated} shows that the AttenFeed module in uViT relaxes the Attention--FFN dichotomy in ViT, and \cref{exp:transfer} demonstrates that this leads to reduced inductive bias compared to ViT. Taken together, these findings indicate that uViT mitigates the inductive bias inherent in the Attention--FFN dichotomy.
Leveraging uViT as a control group for investigating the role of the Attention--FFN separation, we compare the pre-training performance of ViT and uViT to examine under what conditions the Attention--FFN dichotomy is beneficial. Our results show that the benefit of the Attention--FFN dichotomy in ViT diminishes as the model size decreases.

\paragraph{Settings.}
We pre-train the uViT-S/B/L and ViT-S/B/L on three larger-scale datasets: ImageNet-1k, Places365~\cite{places365}, and iNaturalist 2021~\cite{inaturalist2021}. The training procedure for all experiments is identical to the setup described in \cref{exp:attention}.

\paragraph{Results.}
The experimental results are presented in \cref{tab:pretrain}. When the model size is Small, uViT significantly outperforms ViT. However, as the model scales from Small to Base, this performance gap diminishes substantially. Since ViT and uViT were trained under identical conditions, these performance differences arise from the structural dichotomy of Attention--FFN. Therefore, we interpret this diminishing performance gap as an indication that the Attention--FFN dichotomy in ViT becomes counterproductive at a smaller model scales.

\paragraph{Hypothesis on the Result.}
We attribute the disadvantage of the Attention--FFN dichotomy at smaller scales to the rigid parameter allocation of ViT. Under the limited parameter budgets of small models, allocating parameters efficiently becomes critical. However, the strict separation between Attention and FFN forces parameters into fixed functional roles, which can lead to structural inefficiencies. For example, the standard architecture cannot dynamically allocate more capacity toward attention due to the fixed parameter budget assigned to the FFN, even when modeling token relationships is crucial for the dataset. In contrast, uViT provides architectural flexibility that allows limited parameters to be utilized more adaptively.

\begin{center}
\begin{tcolorbox}[
    width=0.48\textwidth,
    colframe=darkgray,
    colback=orange!3,
    boxrule=1pt,
    arc=1mm,
    boxsep=2pt
]
\textbf{Finding of \cref{exp:scaling}}: Attention--FFN dichotomy in ViT becomes detrimental at a smaller model.
\end{tcolorbox}
\end{center}
\section{Conclusion}
In this work, we revisit one of the most fundamental yet under-examined design choices in Vision Transformers: the rigid alternation between Attention and Feed-Forward Networks (FFNs). To analyze the necessity of this Attention--FFN dichotomy, we propose the AttenFeed module, a unified component that mathematically and empirically generalizes both standard attention and FFNs. Building upon this formulation, we introduce the unified Vision Transformer (uViT), which replaces the alternating structure with a sequence of AttenFeed modules, thereby relaxing the structural inductive bias imposed by explicit functional separation. Through theoretical analysis and extensive experiments, we demonstrate that the conventional Attention--FFN dichotomy acts as a restrictive inductive bias at smaller model scales.
The limitations of our study are discussed in \cref{sec:limitation}.

\newpage
\section*{Impact Statement}




The models we analyze are pre-trained on large-scale image datasets (ImageNet-1k, Places365, iNaturalist-2021) known to contain geographic, taxonomic, and demographic imbalances; any downstream use of the resulting representations inherits these biases and should be subject to the fairness and robustness practices already established for pre-trained vision backbones. Beyond this, we are not aware of specific societal consequences of this work that require additional discussion.


\bibliography{icml2026}
\bibliographystyle{icml2026}

\newpage
\appendix
\onecolumn

\clearpage
\section{Mathematical Descriptions}
\label{appendix:math_description}

We now demonstrate the theoretically generalized formulation of AttenFeed modules encompassing both standard attention and FFN components, thus minimizing the inductive biases.

\subsection{FFN-like Behavior.} First, we examine the asymptotic behavior of the module when the attention mechanism becomes highly localized. We demonstrate that when each token attends exclusively to itself, the spatial mixing vanishes, and the module collapses into a standard, position-wise FFN.
\begin{proposition}
If the {\small $h$}-th AF head at layer {\small $\ell$} is self-focused, then the head's output approximates a position-wise FFN: {\small $\text{AF}^{\ell, h}(\mathbf{X}^{\ell-1})_i \approx \text{FFN}(\mathbf{x}_i^{\ell-1})$}.
\end{proposition}
\begin{proof}
First, we must show that AttenFeed module follows the condition of self-focus, such that {\small $a_{i,i}^{\ell, h} \to 1$} and $a_{i,j}^{\ell, h} \to 0$ for all {\small $j \neq i$}. The possibility of this condition is empirically established by various studies\cite{bertattention, no-op, deepvit, zebra, dominance} on Transformer models called diagonal dominance, and as AttenFeed module is based on Transformer, the condition is given.
Recall the definition of our module for a single token {\small $i$} as the sum of attention-weighted activation units:
\begin{equation}
\small
    \text{AF}^{\ell, h}(\mathbf{X}^{\ell-1})_i = \sum_{j} a_{i,j}^{\ell, h} \text{GELU}\left( \mathbf{x}_j^{\ell -1} \mathbf{W}_V^{\ell, h} \right) \mathbf{W}_O^{\ell, h}.
\end{equation}
We decompose this summation into the diagonal term ({\small $j=i$}) and the off-diagonal terms ({\small $j \neq i$}):
{\small
\begin{align}
    \text{AF}^{\ell, h}(\mathbf{X}^{\ell-1})_i &= \underbrace{a_{i,i}^{\ell, h} \text{GELU}\left(\mathbf{x}_i^{\ell -1} \mathbf{W}_V^{\ell, h}\right) \mathbf{W}_O^{\ell, h}}_{\text{Diagonal Term}} \\
    &~~~~+ \underbrace{\sum_{j \neq i} a_{i,j}^{\ell, h} \text{GELU}\left(\mathbf{x}_j^{\ell -1} \mathbf{W}_V^{\ell, h}\right) \mathbf{W}_O^{\ell, h}}_{\text{Off-diagonal Term}}.
\end{align}
}
Applying the diagonal dominance condition where {\small $a_{i,i}^{\ell, h} \approx 1$} and {\small $a_{i,j}^{\ell, h} \approx 0$} for {\small $j \neq i$}, the cross-term vanishes:
{\small
\begin{align}
    \text{AF}^{\ell, h}(\mathbf{X}^{\ell-1})_i &\approx 1 \cdot \text{GELU}\left(\mathbf{x}_i^{\ell -1} \mathbf{W}_V^{\ell, h}\right) \mathbf{W}_O^{\ell, h} + \sum_{j \neq i} 0 \cdot (\dots) \\
    &= \text{GELU}\left(\mathbf{x}_i^{\ell -1} \mathbf{W}_V^{\ell, h}\right) \mathbf{W}_O^{\ell, h}.
\end{align}
}
This resulting expression is exactly the definition of a standard, token-wise GELU FFN applied to {\small $\mathbf{x}_i^{\ell-1}$}, proving that the AttenFeed module naturally reduces to a standard FFN structure.
\end{proof}

\subsection{Attention-like Behavior.} To mathematically justify our module's behavior in reproducing the standard attention module, we utilize the argument that an FFN layer with SiLU activation is mathematically identical to the attention head of the standard module~\cite{math}. For more details on this, refer to Theorem 4 of the cited study. Our argument can be summed up with the following proposition.

\begin{proposition}
Assume the pre-activation inputs to the GELU function are sufficiently small, such that {\small $\| \mathbf{x}_j^{\ell-1} \mathbf{W}_V^{\ell, h} \| \to 0$} for all {\small $j$}. Under this condition, the AF head approximates a standard attention head: {\small $\text{AF}^{\ell, h}(\mathbf{X}^{\ell-1}) \approx \text{Attn}^{\ell,h}(\mathbf{X}^{\ell-1})$}.
\end{proposition}

\begin{proof}
Let the pre-activation vector be {\small $\mathbf{z}_{j} = \mathbf{x}_j^{\ell-1} \mathbf{W}_V^{\ell, h}$}. Recent studies on activation sparsity demonstrate that modern architectures exhibit highly sparse, input-dependent activation patterns even with non-strict activation functions like GELU \cite{xwvzero1, xwvzero2, xwvzero3, xwvzero4, xwvzero5}. Consequently, the components of {\small $\mathbf{z}_{j}$} can be partitioned into two distinct functional regimes: a sparse subset of large positive activations and a dense majority of near-zero values. We analyze the behavior of the GELU function across these two regimes:
\begin{itemize}
\item[$\circ$] \textbf{Sparse active regime ({\small $z \gg 0$}):} The function approaches its linear asymptote, {\small $\text{GELU}(z) \approx z$}~\cite{gelustudy}. These few large values dominantly drive the output of the layer. The massive activation phenomenon also falls into this category~\cite{massiveactivation}.
\item[$\circ$] \textbf{Dense inactive regime ({\small $z \approx 0$}):} Following recent observations on activation sparsity~\cite{sparsegpt}, the activation values in this region are sufficiently small and inherently negligible. Thus, we can directly approximate the function as near-zero, i.e., {\small $\text{GELU}(z) \approx 0$}.
\end{itemize}
For the inactive regime, since both {\small $z$} and {\small $\text{GELU}(z)$} are approximately zero, replacing {\small $\text{GELU}(z)$} with {\small $z$} introduces near-zero absolute error. Assuming large negative activations are effectively suppressed by the sparsity pattern, we can approximate the activation function as an identity mapping over the entire vector without significant loss of accuracy:
\begin{equation}
\small
    \text{GELU}(\mathbf{z}_{j}) \approx \mathbf{z}_{j}.
\end{equation}
Substituting this linear approximation back into the AttenFeed module formulation yields:
\begin{equation}
\small
    \text{AF}^{\ell, h}(\mathbf{X}^{\ell-1})_i \approx \sum_{j} a_{i,j}^{\ell, h} \left( \mathbf{x}_j^{\ell -1} \mathbf{W}_V^{\ell, h} \right) \mathbf{W}_O^{\ell, h} = \text{Attn}^{\ell, h}(\mathbf{X}^{\ell-1}).
\end{equation}
This resulting expression is mathematically identical to the formal definition of the standard attention mechanism given in Equation \ref{eqn:attention}. Thus, driven by the properties of activation sparsity, the AttenFeed module structurally converges to standard attention. \qed
\end{proof}

\begin{proposition}
Let an AF head $\text{AF}^{\ell, h}$ follow the conditions of \cref{prop1}. Then, there exists a set of parameterizations for the AF head such that {\small $\text{AF}^{\ell,h}(\mathbf{X}^{\ell-1}) \approx \text{Attn}^{\ell,h}(\mathbf{X}^{\ell-1})$}.
\end{proposition}

\begin{proof}
From \cref{prop1}, we can redefine AttenFeed modules as FFN with GELU activation. To utilize the mentioned argument, we must justify the cases where GELU function approximates SiLU function. This can be done with a simple minimax approximation. We define the native Sigmoid gate as $\sigma$ and establish the following minimax optimization problem to find the optimal scaling factor {\small $\lambda^*$} that maps the Gaussian CDF {\small $\Phi$} to the standard Sigmoid:
\begin{equation}
\small
    \lambda^* = \underset{\lambda}{\arg \min} \left( \max_{x \in \mathbb{R}} | \Phi(\lambda x) - \sigma(x) | \right).
\end{equation}
Based on the minimax theorem, such {\small $\lambda^*$} exists (numerically, {\small $\lambda^* \approx 1.702$}~\cite{lambda-asterisk}) and can be parameterized into AttenFeed modules from the perspective of FFN. Without loss of generality, we define {\small $\mathbf{W}_{\text{out}}^{\ell, h, k} \coloneq \lambda^* \mathbf{W}_O^{\ell, h}[k, :] \in \mathbb{R}^{d}$} and {\small $\mathbf{W}_{\text{in}}^{\ell, h, k} \coloneq \mathbf{W}_V^{\ell, h}[:, k] \in \mathbb{R}^{d}$} where {\small $k \in \{1,\ldots,d_h\}$}, and rewrite the AttenFeed module as follows: 
{\small
\begin{align}
    \text{AF}^{\ell, h}(\mathbf{X}^{\ell-1})_i &  = \text{GELU}\left(\mathbf{x}_i^{\ell -1} \mathbf{W}_V^{\ell, h}\right) \mathbf{W}_O^{\ell, h} \\ &
    \approx \sum^{d_h}_{k=1}\text{SiLU}\left(\mathbf{x}_i^{\ell -1} \mathbf{W}_\text{in}^{\ell, h, k}\right) \mathbf{W}_\text{out}^{\ell, h, k}.
\end{align}
}
Under the assumption of diagonal dominance, the conditions of Theorem 4 in \cite{math} are satisfied. Thus, the following holds:\footnote{Concretely, Theorem~4 in \cite{math} assumes a valid masked attention mechanism where each row contains at least one mask. Under our assumption of the diagonal dominance property, our framework naturally satisfies the conditions required by the theorem.}
\begin{equation}
\small
    \text{SiLU}\left(\mathbf{x}_i^{\ell -1} \mathbf{W}_\text{in}^{\ell, h, k}\right) \mathbf{W}_\text{out}^{\ell, h, k} = \sum_{j} a_{i,j}^{\ell, h, k} \mathbf{x}_j^{\ell -1} \mathbf{W}_{\text{in}}^{\ell, h, k}\mathbf{W}_\text{out}^{\ell, h, k}.
\end{equation}
This leads to the conclusion that the head of AttenFeed modules act as a standard attention head modules that has {\small $d_h$} heads. This can be concluded with the following equation:
{\small
\begin{align}
    \text{AF}^{\ell, h}(\mathbf{X}^{\ell-1})_i &= \sum_k\sum_{j} a_{i,j}^{\ell, h, k} \mathbf{x}_j^{\ell -1} \mathbf{W}_{\text{in}}^{\ell, h, k}\mathbf{W}_\text{out}^{\ell, h, k} \\
    &= \sum_k \text{Attn}^{\ell, h, k}(\mathbf{X}^{\ell-1}) \\
    &= \text{Attn}^{\ell, h}(\mathbf{X}^{\ell-1}).
\end{align}
}
\end{proof}

\subsection{Mixture of FFN-like Behavior.}
While AttenFeed modules can function as either standard attention or standard FFNs, they also manifest a novel emergent behavior: a spatially weighted mixture of FFNs. As shown in \cref{eqn:mixture-of-ffn}, the extent to which the non-linear FFN outputs from each token position are integrated into the residual stream is dynamically governed by the QK circuit of the attention mechanism.
{\small
    \begin{align}
    \label{eqn:mixture-of-ffn}
        \text{AF}^{\ell, h}(\mathbf{X}^{\ell-1})_i &= \sum_{j} a_{i,j}^{\ell, h} \sigma(\mathbf{x}_j^{\ell -1} \mathbf{W}_V^{\ell, h}) \mathbf{W}_O^{\ell, h} \\
        &= \sum_{j} a_{i,j}^{\ell, h} \text{FFN}(\mathbf{x}_j^{\ell -1}).
    \end{align}
}

\clearpage
\section{Training Details}
\label{appendix:exp_setting_detail}
We evaluate the proposed unified Vision Transformer (\textbf{uViT}) against the standard Vision Transformer (\textbf{ViT}) baseline at two scales:
(i) a \emph{toy-scale} regime on small classification datasets used for our rank analyses (\cref{exp:ffn}), and
(ii) a \emph{large-scale} pre-training regime on ImageNet-1k, Places365, and iNaturalist-2021.
In both regimes, every hyper-parameter that is not intrinsic to the architectural difference between ViT and uViT is held fixed, so that all observed differences can be attributed to the architectural change rather than to the optimization recipe.

\subsection{Toy-Scale Training}
\label{sec:training_toy}

\paragraph{Datasets and preprocessing.}
Toy-scale experiments are conducted on CIFAR-10, CIFAR-100, STL-10, and SVHN. All images are resized to a common spatial resolution ($32{\times}32$ for CIFAR-10/100 and SVHN, $96{\times}96$ for STL-10) and patchified with a patch size of $4$.
Training images are augmented with a random crop (padding $4$ for CIFAR-10/100, STL-10 and padding $2$ for SVHN), random horizontal flip (except SVHN, where flipping alters semantics), and RandAugment~\cite{randaug} (CIFAR-10/100, STL-10 only).
Evaluation images are only resized and normalized using each dataset's mean and standard deviation. Identical augmentation pipelines are applied to ViT and uViT.

\paragraph{Optimization.}
All toy-scale models are trained for $100$ epochs with a global batch size of $128$ on a single GPU.
We use the AdamW~\cite{adamw} optimizer with a peak learning rate of $5\!\times\!10^{-4}$ and weight decay $10^{-3}$, warmed up linearly from $\mathrm{peak}/W$ to the peak over $W=10$ epochs and then decayed by a cosine annealing schedule to $10^{-5}$ over the remaining $190$ epochs.
Loss is cross-entropy over hard labels. Dropout is fixed at $0.1$ throughout.
Random seeds, data splits, dataloader ordering, optimizer, learning-rate schedule, dropout, and augmentation are identical across ViT, uViT, and noFFN within each dataset, ensuring that only the architectural choice varies between the compared models.

\subsection{Large-Scale Pre-training}
\label{sec:training_large}

\paragraph{Datasets and preprocessing.}
Large-scale pre-training is performed on ImageNet-1k, Places365-standard, and iNaturalist-2021 at an input resolution of $224{\times}224$ with a patch size of $16$. For training, we apply a standard strong-augmentation pipeline: \textsc{RandomResizedCrop}$(224,\,\text{scale}=(0.08,1.0))$, random horizontal flip with probability $0.5$, and RandAugment with magnitude $9$, followed by normalization with the ImageNet mean and standard deviation.
At the collate stage, each mini-batch is randomly transformed by either MixUp ($\alpha_\text{mixup}=1.0$)~\cite{mixup} or CutMix ($\alpha_\text{cutmix}=1.0$)~\cite{cutmix}, yielding soft labels. Evaluation uses a single center crop: images are first resized to $256$ pixels (short-side) and then center-cropped to $224{\times}224$. The \emph{same}  training and evaluation pipeline is used for ViT and uViT on every dataset.

\paragraph{Optimization.}
Both ViT and uViT are trained with the AdamW optimizer for $300$ epochs.
We adopt a reference recipe of peak learning rate $3\!\times\!10^{-3}$ at a reference global batch of $4096$, and use the linear scaling rule~\cite{goyal2017accurate}:
{\small
    \begin{align}
        \mathrm{lr} \;=\; 3\!\times\!10^{-3} \cdot \tfrac{B_{\text{global}}}{4096},
    \end{align}
}
to adapt to the actual global batch size {\small $B_{\text{global}}$}. A linear warm-up over the first $10{,}000$ optimization steps is followed by a cosine decay to zero over the remainder of training.
Weight decay is $0.3$, gradients are clipped to a global $\ell_2$ norm of $1.0$, and training is conducted in mixed precision (FP16 autocast with a dynamic loss scaler).
The loss is cross-entropy against the soft targets produced by MixUp/CutMix.

\paragraph{Distributed training.}
All large-scale runs use PyTorch \textsc{DistributedDataParallel} over a single multi-GPU node with the NCCL backend, launched via \texttt{torchrun}. A \texttt{DistributedSampler} ensures that every GPU sees a disjoint shard of the training set each epoch. Random seeds are offset by the process rank so that augmentation is decorrelated across GPUs while remaining reproducible. Per-GPU batch sizes are chosen so that $B_{\text{global}}$ is identical between matched ViT/uViT runs, and the linear scaling rule above then yields the same effective learning-rate trajectory for both models.

\clearpage
\section{Experiment Result Details}
\label{appendix:exp_result_detail}

\subsection{Rank Measurement Procedure}
\label{appendix:rank_procedure}

The experiments in \cref{exp:ffn} and \cref{appendix:rank_preservation} rely on a representation rank that is estimated from finite samples rather than from the exact matrix rank. This subsection details the mathematical procedure and clarifies the two rank quantities reported in the tables.

\paragraph{Effective rank via truncated variance.}
Let {\small $\{\mathbf{z}_i\}_{i=1}^{N} \subset \mathbb{R}^{d}$} denote the \texttt{[CLS]} token representations produced by a frozen model at a fixed layer when applied to the $N$ samples of the evaluation set, and let {\small $\mathbf{Z} \in \mathbb{R}^{N \times d}$} be their row-stacked matrix. Applying singular value decomposition (SVD) to $\mathbf{Z}$ yields non-increasing singular values {\small $\sigma_1 \geq \sigma_2 \geq \cdots \geq \sigma_k \geq 0$} with {\small $k = \min(N, d)$}. Because the total variance carried by {\small $\mathbf{Z}$} is {\small $\sum_{i=1}^{k} \sigma_i^{2}$}, the proportion of variance captured by the top-$r$ singular directions is
\begin{equation}
\label{eq:cumvar}
\eta(r) \;=\; \frac{\sum_{i=1}^{r}\sigma_{i}^{2}}{\sum_{i=1}^{k}\sigma_{i}^{2}}.
\end{equation}
Following \citet{rankmeasure1, rankmeasure2}, we define the \emph{effective rank} at a variance-explanation threshold {\small $\tau \in (0,1]$} as the smallest number of singular directions whose cumulative variance reaches {\small $\tau$}:
\begin{equation}
\label{eq:effrank}
r_{\tau}(\mathbf{Z}) \;=\; \min\bigl\{\,r \in \{1,\ldots,k\} \;:\; \eta(r) \geq \tau \,\bigr\}.
\end{equation}
This estimator is a soft analogue of the exact matrix rank that discards singular directions carrying negligible variance and is therefore robust to numerical and sampling noise. We fix {\small $\tau = 0.95$} throughout, following the convention used in prior works on rank analysis of deep networks.

\paragraph{Normalized rank.}
The raw effective rank {\small $r_{\tau}(\mathbf{Z})$} is bounded above by {\small $k = \min(N, d)$}, which depends on the feature dimension of the representation under analysis. To enable a fair comparison between representations with different feature dimensions (e.g., between the embedding features of a model and its logit outputs), we report the \emph{normalized rank}~\cite{normalizedrank}
\begin{equation}
\label{eq:normrank}
\rho_{\tau}(\mathbf{Z}) \;=\; \frac{r_{\tau}(\mathbf{Z})}{\min(N, d)} \;=\; \frac{r_{\tau}(\mathbf{Z})}{d},
\end{equation}
where the last equality uses the fact that $N \gg d$ for every evaluation set considered in this paper. The normalized rank takes values in {\small $(0, 1]$}: a value close to $1$ indicates that essentially every singular direction of the feature space carries a non-negligible amount of variance, whereas a value close to $0$ indicates that the representation has collapsed into a low-dimensional subspace.

\paragraph{Pre-head and post-head rank.}
We apply the procedure above at two distinct points of the forward pass in order to characterize both the \emph{inputs} and the \emph{outputs} of the classification head. Denoting the final LayerNorm immediately before the classifier by {\small $\mathrm{LN}(\cdot)$} and the classification head by {\small $\mathrm{Head}(\cdot)$}:
\begin{itemize}
\item[$\circ$] The \textbf{pre-head rank} is the normalized rank of the representation that \emph{enters} the classification head, i.e.\ the output of the final LayerNorm,
{\small $\mathbf{z}_{i}^{\text{pre}} = \mathrm{LN}(\mathbf{h}_{i}^{L})[\texttt{CLS}] \in \mathbb{R}^{d}$}, where {\small $\mathbf{h}_{i}^{L}$} is the last-layer token sequence for sample $i$. The feature dimension is the embedding dimension $d$, so {\small $\rho_{\tau}(\mathbf{Z}^{\text{pre}}) = r_{\tau}(\mathbf{Z}^{\text{pre}})/d$}.
\item[$\circ$] The \textbf{post-head rank} is the normalized rank of the representation that \emph{leaves} the classification head, i.e.\ the logits,
{\small $\mathbf{z}_{i}^{\text{post}} = \mathrm{Head}\bigl(\mathbf{z}_{i}^{\text{pre}}\bigr) \in \mathbb{R}^{C}$}, where $C$ is the number of classes. The feature dimension is now $C$, so {\small $\rho_{\tau}(\mathbf{Z}^{\text{post}}) = r_{\tau}(\mathbf{Z}^{\text{post}})/C$}.
\end{itemize}
The two quantities probe complementary aspects of the learned representation. The pre-head rank measures how much of the embedding space is actively used by the backbone before classification; a low pre-head rank indicates that the backbone has collapsed distinct tokens into an overly narrow subspace, so that samples from different classes become difficult to separate linearly. The post-head rank, in contrast, measures how effectively the head spreads class-discriminative information across the available logit dimensions; a post-head rank well below $1$ indicates that the logit output effectively operates on fewer than $C$ degrees of freedom, which is a direct symptom of confounded classes. Reporting both rank statistics thus allows us to localize rank collapse either to the backbone or to the classification stage.

\subsection{Rank Preservation of AttenFeed Module}
\label{appendix:rank_preservation}

In \cref{exp:ffn}, we empirically demonstrated on CIFAR-10 and CIFAR-100 that the AttenFeed module inherits the rank-preserving behavior of FFNs. To verify that this finding is not specific to a particular dataset, we repeat the same experimental protocol on two additional small-scale datasets, SVHN~\cite{svhn} and STL-10~\cite{stl10}, which differ considerably from CIFAR in visual statistics (digit crops and high-resolution natural images, respectively). For each dataset, xViT (a ViT variant in which the FFN is removed) and uViT are trained under identical conditions at three matched $(L, H, D)$ configurations, and the normalized ranks of the \texttt{[CLS]} token representations immediately preceding and following the classification head (pre-head and post-head) are measured via SVD with the 95\% variance threshold. The results are summarized in \cref{tab:ffn}.

\begin{table*}[h]
  \caption{\textbf{Performance and representation rank comparison on SVHN and STL10.} uViT consistently matches or exceeds xViT in pre-head rank and test accuracy, confirming that the rank-preserving behavior of the AttenFeed module generalizes beyond CIFAR-10/100.}
  \label{tab:ffn}
  \centering
  \setlength{\tabcolsep}{5.6pt}
  \begin{tabular}{ccccccccc}
    \toprule
    & & & \multicolumn{3}{c}{SVHN} & \multicolumn{3}{c}{STL10} \\
    \cmidrule(lr){4-6} \cmidrule(lr){7-9}
    Model & ($L$, $H$, $D$) & Params. & Test acc. & $\text{Rank}_\text{(pre-head)}$ & $\text{Rank}_\text{(post-head)}$ & Test acc. & $\text{Rank}_\text{(pre-head)}$ & $\text{Rank}_\text{(post-head)}$ \\
    \midrule
        xViT & \multirow{2}{*}{(6, 4, 128)} & \multirow{2}{*}{0.43M} & 94.70\% & 0.320 & 0.8 & 67.15\% & 0.367 & 0.6 \\
        uViT & & & 94.78\% & 0.320 & 0.8 & 69.79\% & 0.421 & 0.7  \\
        \cmidrule(lr){1-3} \cmidrule(lr){4-6} \cmidrule(lr){7-9} 
        xViT & \multirow{2}{*}{(6, 4, 160)} & \multirow{2}{*}{0.67M} & 94.87\% & 0.269 & 0.8 & 66.51\% & 0.325 & 0.7  \\
        uViT & & & 95.00\% & 0.294 & 0.8 & 69.60\% & 0.406 & 0.7  \\
        \cmidrule(lr){1-3} \cmidrule(lr){4-6} \cmidrule(lr){7-9} 
        xViT & \multirow{2}{*}{(12, 4, 128)} & \multirow{2}{*}{0.83M} & 94.85\% & 0.336 & 0.8 & 67.31\% & 0.406 & 0.7  \\
        uViT & & & 95.04\% & 0.352 & 0.8 & 70.51\% & 0.430 & 0.7  \\
    \bottomrule
  \end{tabular}
  \vspace{-5pt}
\end{table*}

Across every configuration, uViT attains a pre-head rank that is equal to or higher than that of xViT, and this higher rank is consistently accompanied by higher test accuracy. The advantage is especially pronounced on STL-10, where uViT improves the pre-head rank by up to $0.081$ and the test accuracy by up to $3.20$ percentage points over xViT. On SVHN, where both models nearly saturate due to the relative simplicity of the task, the rank and accuracy margins are smaller but still favor uViT. These observations replicate the CIFAR-10/100 trend reported in the main text and reinforce our conclusion that the GELU activation in the AttenFeed module imparts FFN-like rank-preserving properties independent of the target dataset.

\clearpage
\subsection{Attention Map Comparison of uViT and ViT}
\label{appendix:qualitative}

To complement the quantitative segmentation scores reported in \cref{exp:attention}, we present additional qualitative comparisons of the attention maps produced by ViT-B and uViT-B in the figure below. Across a wide variety of object categories and scene compositions, uViT consistently attends to the salient object regions identified by ViT, with only minor deviations in spatial extent. This qualitative agreement extends the main-text finding to a broader set of images and further supports Propositions~\ref{prop2} and \ref{prop3}.

\begin{figure*}[h]
  \centering
  \includegraphics[width=0.98\textwidth]{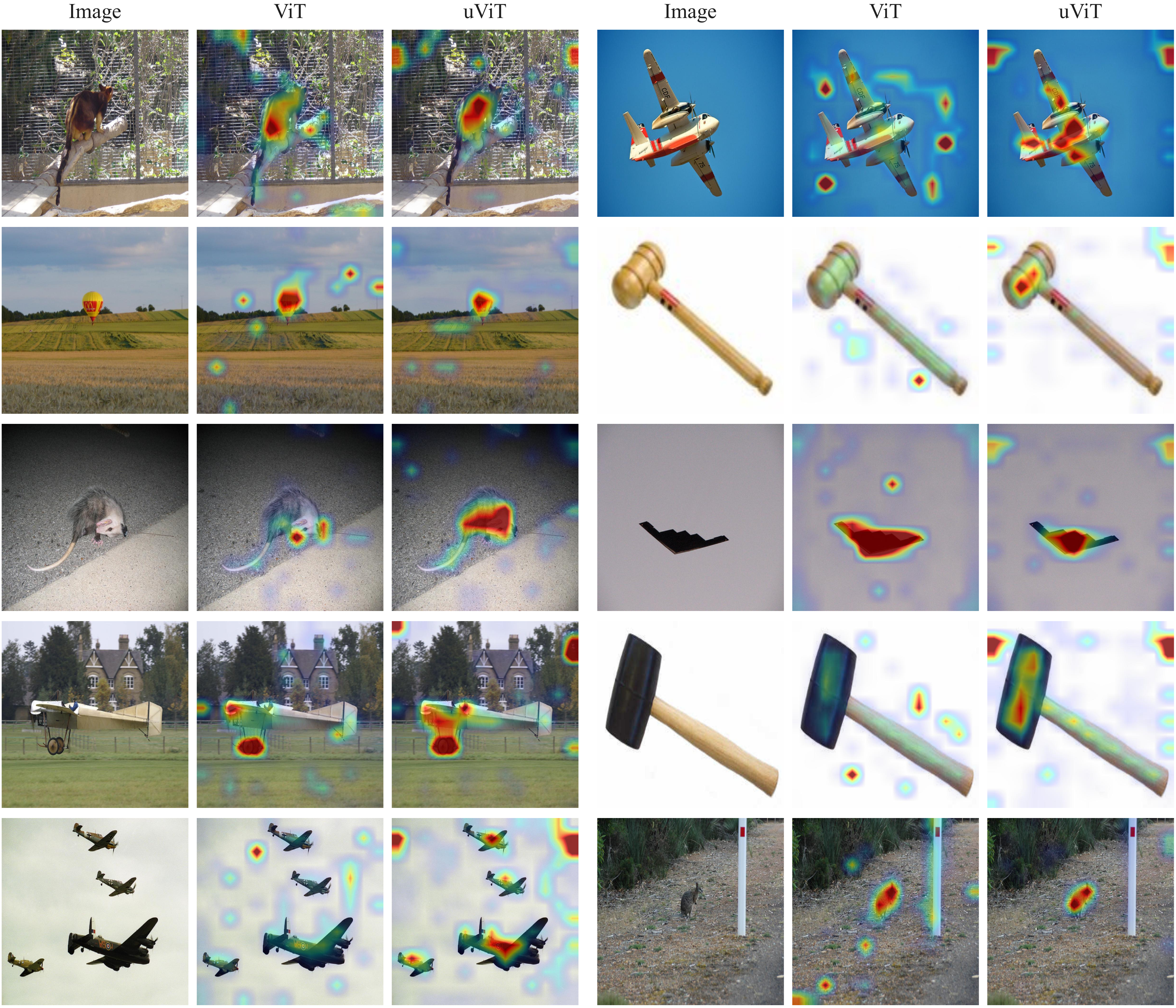}
  \caption{\textbf{Qualitative comparison of attention maps between ViT and uViT.} uViT-B reliably highlights the same salient object regions as ViT-B, supporting the attention-like behavior of the AttenFeed module.}
  \label{fig:attn_qualitative_suppl}
\end{figure*}

\clearpage
\subsection{Parametric Space Comparison of uViT and ViT}
\label{appendix:weight_comparison_suppl}

\cref{exp:weight_comparison} in the main text shows that, for the Base-scale models, the values of {\small $\mathbf{W}_V$} and {\small $\mathbf{W}_O$} in the AttenFeed module lie in an intermediate regime between the corresponding parameters of standard attention ({\small $\mathbf{W}_V$}, {\small $\mathbf{W}_O$}) and those of the FFN ({\small $\mathbf{W}_\text{in}$}, {\small $\mathbf{W}_\text{out}$}) in ViT. Here, we examine whether this transitional behavior also holds at the Small and Large scales. The layer selection, plotting procedure, and normalization are kept identical to the main-text protocol; the only difference is the model scale. The results are shown in the two figures below, corresponding to the Small and Large configurations, respectively.

At both scales, the uViT weight distributions are systematically wider than those of ViT's attention parameters ({\small $\mathbf{W}_V$}, {\small $\mathbf{W}_O$}) yet narrower than those of ViT's FFN parameters ({\small $\mathbf{W}_\text{in}$}, {\small $\mathbf{W}_\text{out}$}), reproducing the intermediate-regime pattern observed for the Base model. The consistency of this pattern across Small, Base, and Large models indicates that the unified character of the AttenFeed module is a scale-agnostic property of the architecture rather than an artifact of a specific capacity. This provides additional empirical support for our claim that the AttenFeed module simultaneously inherits the parametric tendencies of both the attention and FFN components of a standard Transformer.

\begin{figure*}[h]
  \centering
  \includegraphics[width=0.98\textwidth]{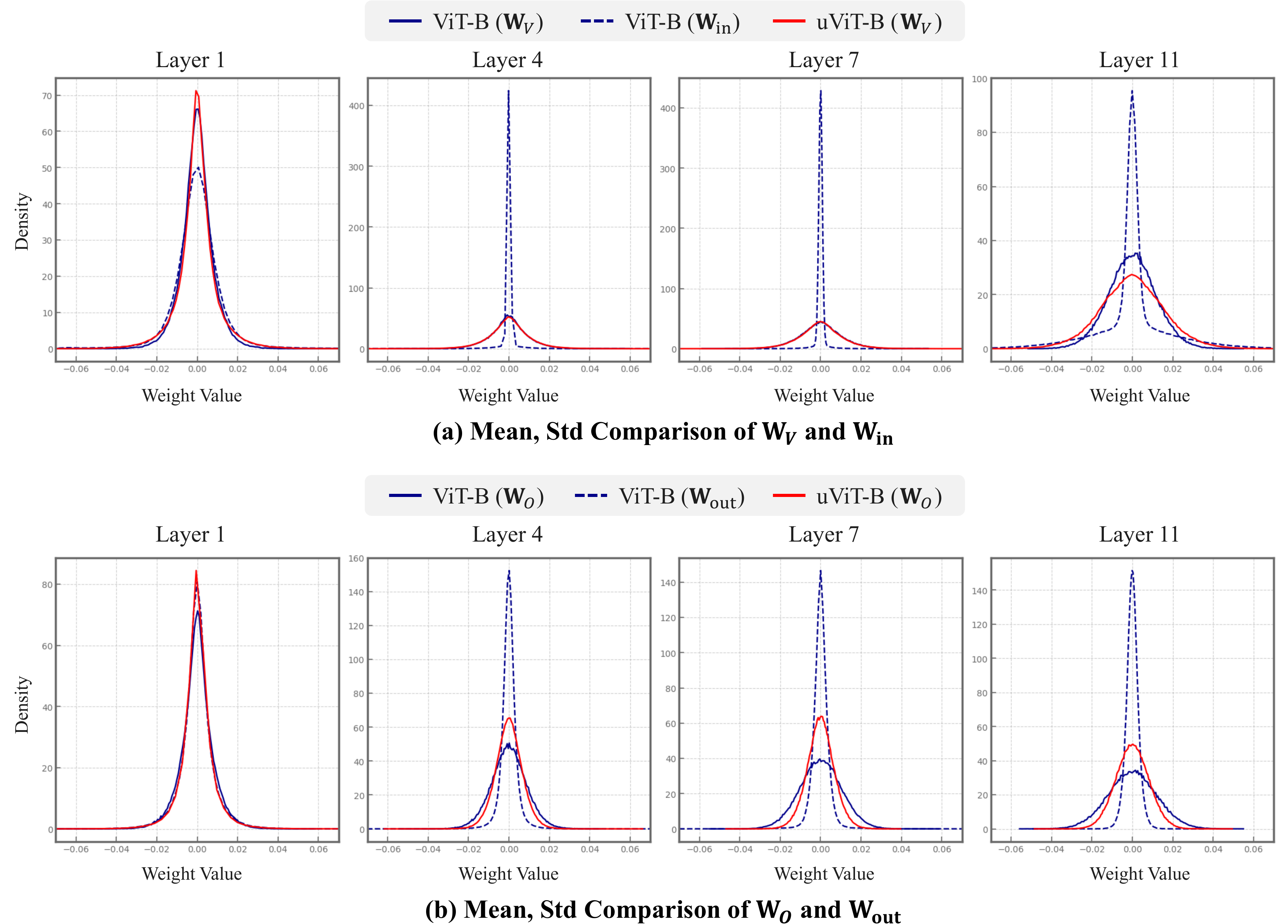}
  \caption{\textbf{Weight value distributions of uViT-S and ViT-S.} The results show that uViT parameters occupy a transitional space between standard attention and FFN regimes at the Small scale, providing empirical evidence for the unified nature of the AttenFeed module.}
  \label{fig:weight_comparison_suppl_S}
  \vspace{-5pt}
\end{figure*}

\begin{figure*}[h]
  \centering
  \includegraphics[width=0.98\textwidth]{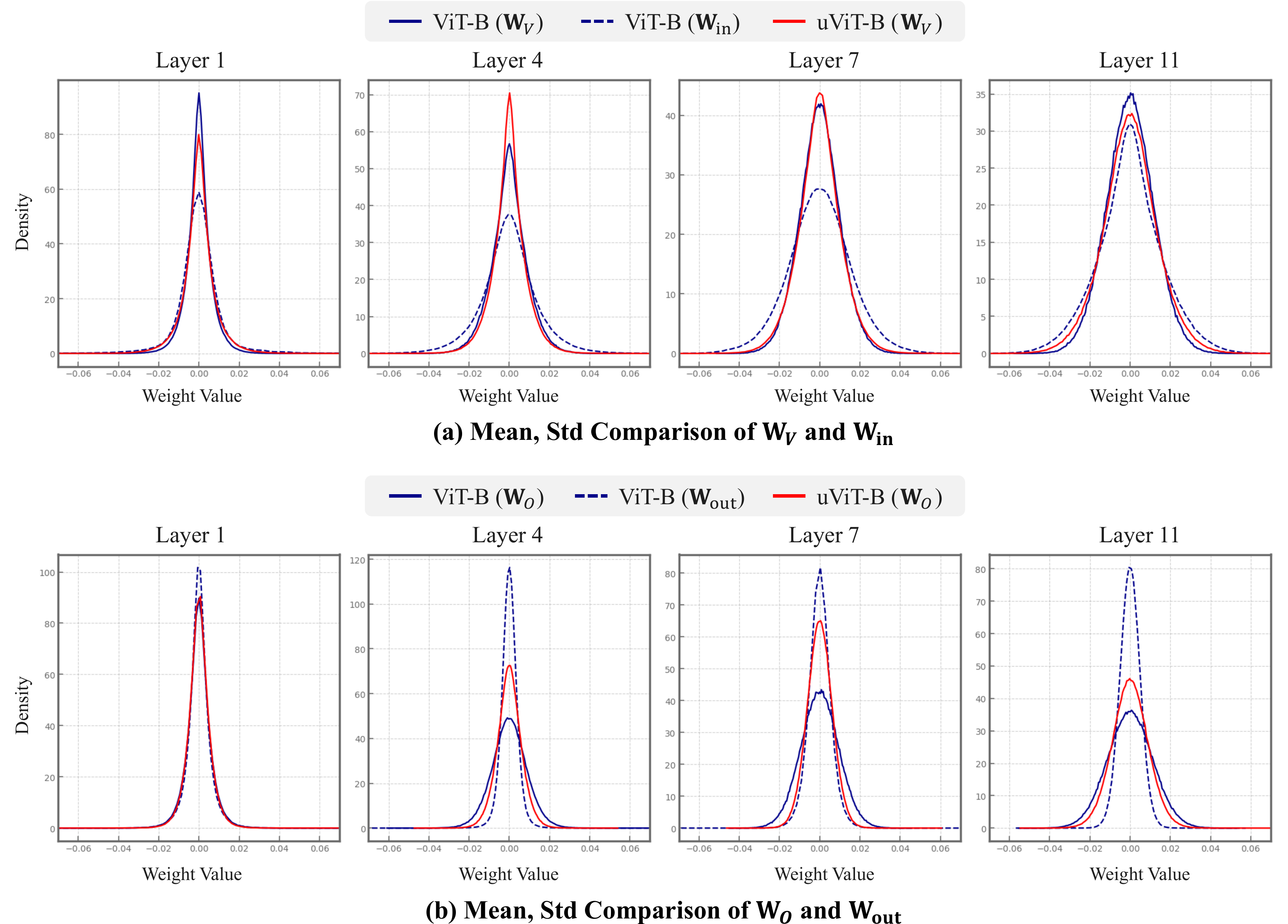}
  \caption{\textbf{Weight value distributions of uViT-L and ViT-L.} The same transitional pattern holds at the Large scale, indicating that the unified nature of the AttenFeed module is not scale-specific.}
  \label{fig:weight_comparison_suppl_L}
  \vspace{-5pt}
\end{figure*}

\clearpage
\section{Limitation}
\label{sec:limitation}

While this work provides theoretical and empirical insights into the Attention-FFN dichotomy through the proposed AttenFeed module and uViT architecture, we acknowledge a few limitations that warrant future investigation.

First, our empirical analysis is currently confined to the computer vision domain. Although the Transformer was originally designed for natural language processing and has since become a universal architecture across various modalities, our experiments strictly evaluate Vision Transformers (ViTs) on image datasets. The inductive biases imposed by the Attention-FFN separation may manifest differently in other domains, such as language modeling or audio processing. Validating the necessity of the Attention-FFN dichotomy and the efficacy of the AttenFeed module in these non-vision domains remains an important direction for future work.

Second, to ensure a strictly fair comparison between uViT and the standard ViT, all models in our experiments were trained under identical conditions using a standard ViT training recipe. However, since uViT fundamentally alters the conventional network structure by unifying Attention and FFN into a single sequence of AttenFeed modules, it may possess a distinct optimal optimization landscape. Developing a specialized training recipe tailored to uViT---such as exploring different regularization techniques, learning rate schedules, or data augmentations---could potentially unlock further performance gains and alter the scaling behaviors observed in our current findings.

\end{document}